\documentclass[conference]{IEEEtran}
\IEEEoverridecommandlockouts
\usepackage{cite}
\usepackage{amsmath,amssymb,amsfonts}
\usepackage{graphicx}
\usepackage{textcomp}
\usepackage{xcolor}

\usepackage{amsthm}
\newtheorem{theorem}{Theorem}
\usepackage{algorithm}
\usepackage{algpseudocode}
\usepackage{enumerate}
\usepackage{booktabs}
\usepackage[hidelinks]{hyperref}

\def\BibTeX{{\rm B\kern-.05em{\sc i\kern-.025em b}\kern-.08em
    T\kern-.1667em\lower.7ex\hbox{E}\kern-.125emX}}
\begin{document}

\title{Exploring Weaknesses of Generative Image Watermarks against Latent Frequency Masking\\


\thanks{The work of K. Abud, A. Yakushev, A. Akimenkov, D. Obydenkov, Y. Markin and A. Antsiferova was supported by a grant, provided by the Ministry of Economic Development of the Russian Federation (agreement dated June 20, 2025 No. 139-15-2025-011, identifier 000000C313925P4G0002).}
}

\author{
\IEEEauthorblockN{Kirill Aistov\IEEEauthorrefmark{1}, Khaled Abud\IEEEauthorrefmark{2}\IEEEauthorrefmark{1}, Irina Serzhenko, Egor Kovalev\IEEEauthorrefmark{1},\\ Aleksey Yakushev\IEEEauthorrefmark{2}, Aleksandr Akimenkov\IEEEauthorrefmark{2}, Dmitry Obydenkov\IEEEauthorrefmark{2}, Yury Markin\IEEEauthorrefmark{2},\\ Sergey Lavrushkin\IEEEauthorrefmark{1}, Dmitriy Vatolin\IEEEauthorrefmark{1} and Anastasia Antsiferova\IEEEauthorrefmark{2}\IEEEauthorrefmark{1}}
\IEEEauthorblockA{\IEEEauthorrefmark{1}MSU AI Institute, Moscow, Russia\\
Email: aistovkk@my.msu.ru, \{khaled.abud, egor.kovalev, sergey.lavrushkin, dmitriy\}@graphics.cs.msu.ru}
\IEEEauthorblockA{\IEEEauthorrefmark{2}Trusted AI Research Center RAS, Moscow, Russia\\
Email: \{h.abud, yakushev, a.akimenkov, obydenkov, ustas, aantsiferova\}@ispras.ru}
}

\IEEEoverridecommandlockouts
\IEEEpubid{\makebox[\columnwidth]{\copyright2026 IEEE \hfill}
\hspace{\columnsep}\makebox[\columnwidth]{ }}

\maketitle

\IEEEpubidadjcol

\begin{abstract}
Invisible watermarking has become a central tool for tracing AI-generated images, but its robustness against adaptive removal attacks remains an open security question. We introduce Latent
Frequency Masking, an attack that erases watermark evidence by replacing selected Fourier coefficients in the latent representation of a watermarked image. The replacement can be
sampled from Gaussian noise for efficiency or derived from diffusion regeneration for improved image preservation. We provide a theoretical distortion bound relating the change between the reconstructed adversarial image and the masked latent-frequency perturbation. We evaluate the proposed attack against six diffusion watermarking methods on images generated from DiffusionDB and MS-COCO prompts. Latent Frequency Masking removes or substantially weakens several watermarks while preserving perceptual quality and achieving favorable runtime compared with existing attacks. These results identify latent-frequency manipulation as a practical attack surface and highlight the need to include such
attacks in robustness evaluations of generative image watermarking.
\end{abstract}


\section{Introduction}
The rapid advancement of generative AI has started a new era of photorealistic image synthesis. Modern diffusion-based and autoregressive models can produce images of remarkable quality, making it difficult to distinguish synthetic content from authentic photographs. While these capabilities unlock significant creative and commercial potential, they also create critical societal challenges: preventing AI-generated misinformation, deepfakes, and fabricated visual content.
\begin{figure}[tbp]
    \centerline{\includegraphics[width=0.85\columnwidth]{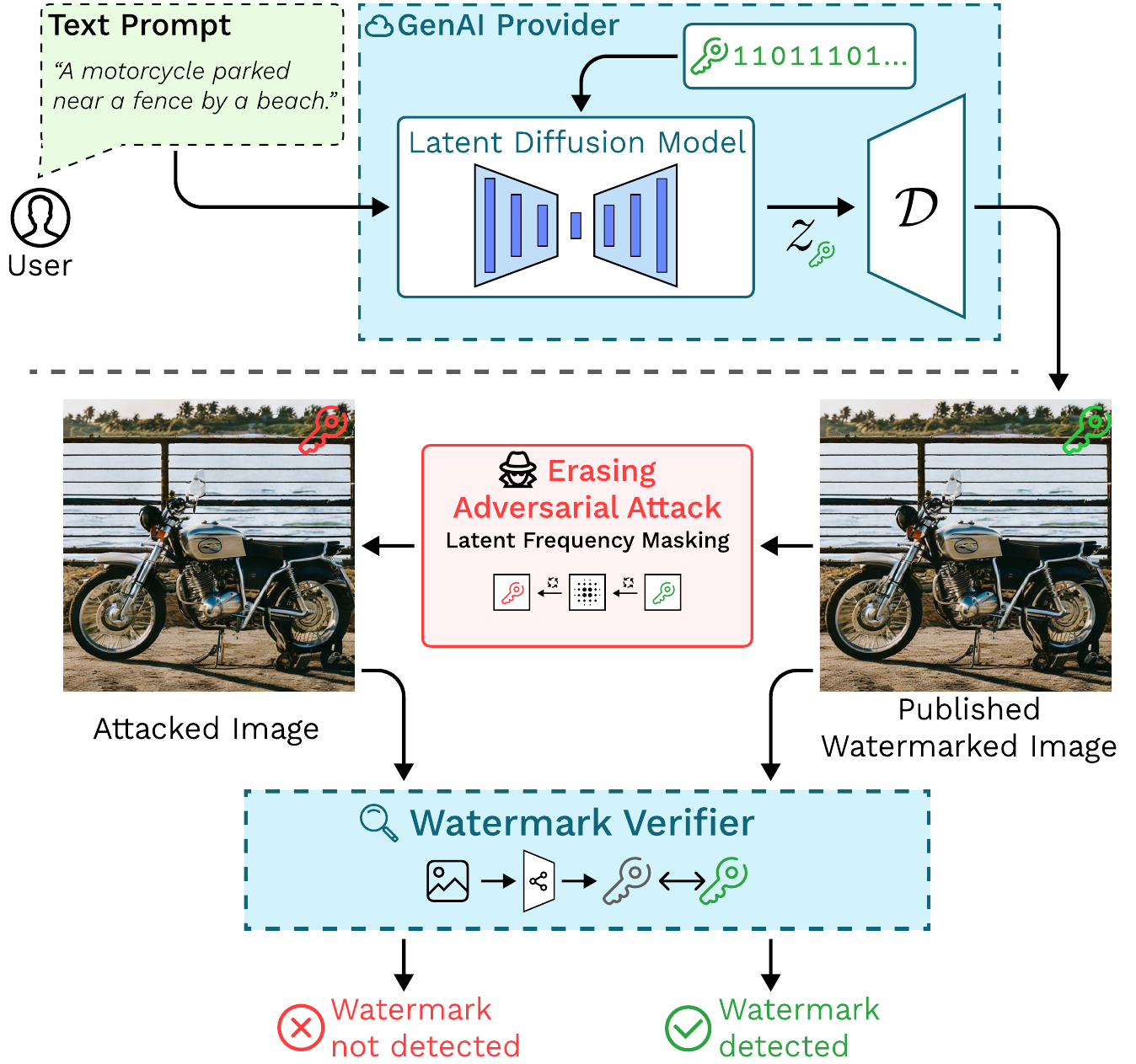}}
    \caption{Overview of the watermark robustness evaluation setup}
\label{fig:wm_scheme}
\vspace{-1.5em}
\end{figure}
Detecting generated images through empirical classifiers has proven insufficient --- such models tend to overfit to the artifacts of specific generators and fail to generalize to unseen architectures \cite{zheng2024breaking}. As a result, digital watermarking has emerged as the main paradigm for tracking AI-generated content: embedded imperceptible marks at the time of generation can later indicate which generation model has been used and even the exact user session. The urgency of this approach has been recognized at the policy level~\cite{act_overview}: 
the EU and China now mandate detectable marking of AI outputs, while the US is moving toward similar requirements through state law and proposed federal legislation.

The practical security of watermarking schemes remains an open question. As new watermarking methods are proposed, new attacks are designed to remove or corrupt the embedded signal. Figure 1 illustrates this setting: a provider embeds an imperceptible watermark at generation time, after which an adversary may attempt to erase it before publication, while a verifier seeks to recover the mark from the published image. A substantial body of work has examined the vulnerability of spatial-domain and learning-based watermarks; however, one class of methods has received comparatively little scrutiny: frequency-based watermarking schemes that encode information in an image's spectral representation.

In this paper, we expose a previously unreported vulnerability in this family of methods. We demonstrate that frequency-domain watermarks can be systematically removed via an adversarial attack that remains visually imperceptible, preserving image quality while fully neutralizing the embedded signal. To our knowledge, we are the first to systematically exploit latent-frequency masking against this class. We further discuss potential defense strategies and analyze their trade-offs.
Our main contributions are as follows:
\begin{itemize}
\item We identify and formalize a novel attack surface in generative image watermarking: the \textbf{latent-frequency representation} obtained by applying the Fourier transform to a diffusion model's VAE latent.
\item We propose \textbf{Latent Frequency Masking (LFM)}, a gradient-free removal attack that suppresses watermark evidence by replacing selected latent-frequency coefficients, with two replacement strategies --- efficient Gaussian sampling and image-adaptive diffusion regeneration.
\item We derive a distortion bound relating the pixel-space change to the VAE reconstruction error and the masked latent-frequency perturbation, characterizing the quality–removal trade-off.
\item We conduct extensive experiments against six diffusion watermarking schemes on DiffusionDB and MS-COCO, showing that LFM removes or substantially weakens most while preserving perceptual quality and achieving favorable runtime, and we analyze why distributed embedding (e.g., Gaussian Shading) resists it.
\end{itemize}

\section{Related work}
\subsection{Invisible Watermarking Methods}

Digital image watermarking techniques can be categorized by the stage at which the watermark is embedded. The following classification of methods is proposed:
\begin{enumerate}
    \renewcommand{\theenumi}{\roman{enumi}}
    \item post-generation methods, which apply a watermark to already generated or captured images.
    \item in-generation methods, which integrate the watermark directly into the image creation process,
\end{enumerate}

\textbf{Post-generation methods}, which are applied to already existing images, traditionally operate in the spatial, discrete cosine, Fourier, or wavelet domains. These methods embed watermarks into selected frequency coefficients~\cite{poljicak2011discrete, al2007combined, islam2020svm}. Contemporary neural methods adhere to the 'encoder — noise layer — decoder' architecture pioneered by HiDDeN~\cite{zhu2018hidden}, frequently incorporating invertible networks~\cite{ma2022towards}, adversarial training, and attention mechanisms to enhance imperceptibility and robustness~\cite{huang2023arwgan, zhang2019robust, fernandez2024video, souvcek2025pixel}. To demonstrate resilience against geometric attacks, recent works also explore multi-region embedding~\cite{guo2023practical, hu2025mask} and synchronization patterns~\cite {fernandez2025geometric}.

\textbf{In-Generation methods} have become the dominant approach for AI-generated content provenance. The majority of these models operate by modifying latent representation, either before or during the diffusion process. 

Tree-Ring~\cite{wen2023tree} embeds a structured ring-shaped pattern in the frequency domain of the initial noise, achieving strong robustness to geometric transformations. Subsequent works~\cite{ci2024ringid, varlamov2024metr} extended this idea with discretized ring patterns and multi-bit encoding. MaXsive~\cite{mao2025maxsive} injects a structured X-shaped template into the latent noise to enable geometric alignment, while encoding the watermark message in separate latent components. Gaussian Shading~\cite{yang2024gaussian} preserves the statistical properties of the noise and embeds the watermark by modulating latent variables according to predefined quantile intervals of the normal distribution. Stable Signature~\cite{fernandez2023stable} adopts a different strategy by fine-tuning the generative model itself, specifically the VAE decoder, to imprint a watermark into the generated latent. 

Compared with post-generation watermarks, in-generation methods offer several practical advantages that motivate our focus. By weaving the mark into the sampling trajectory rather than adding it afterward, several designs achieve minimal perceptual degradation — e.g., Gaussian Shading~\cite{yang2024gaussian} is provably distribution-preserving — while exhibiting markedly stronger robustness to conventional distortions than spatial-domain post-hoc schemes; standardized stress tests~\cite{an2024waves} confirm that methods like Tree-Ring~\cite{wen2023tree} withstand perturbations on which post-hoc watermarks usually fail. This combination of quality and resilience is precisely what makes their security against adaptive, watermark-aware removal an open and consequential question — the one we examine in this work.

\subsection{Attacks on Invisible Watermarks}

The robustness of watermarking schemes is determined not only by resistance to benign image distortions, but also by the strength of the attacker considered during evaluation. Standard perturbations such as JPEG compression, blur, noise, cropping, and geometric transforms remain necessary sanity checks, yet recent work shows that robustness to such distortions does not imply security against adaptive removal or forgery.

One line of work studies analytical attacks that estimate the watermark directly. Yang et al.~\cite{yang2024steganalysisdigitalwatermarkingdefense} show that content-agnostic watermark patterns can often be recovered by averaging multiple watermarked images and subtracting the estimated pattern, and that this estimate may even enable forgery. A second major family comprises regeneration-based attacks. Zhao et al.~\cite{zhaoInvisibleImageWatermarks2024} demonstrate that invisible watermarks can be removed by first perturbing an image and then reconstructing it with a generative prior, which motivates diffusion-based baselines such as Diffusion Regeneration and Diffusion Rinsing attacks used in our experiments. Related diffusion purification and model-substitution attacks were analyzed by Saberi et al.~\cite{saberiRobustnessAIImageDetectors2023}. Liang et al.~\cite{liangBaselineMethodRemoving2025} further show that Deep Image Prior can serve as a competitive black-box baseline, while Alam et al.~\cite{alam2025saliency} improve the quality--removal trade-off with saliency-aware diffusion reconstruction (SADRE).

Another line of work studies adaptive and optimization-based removal. Lukas et al.~\cite{lukasLeveragingOptimizationAdaptive2023} formulate watermark evasion as an optimization problem and show that optimized attacks can substantially outperform fixed perturbation baselines. UnMarker~\cite{kassis2025unmarker} further argues that robust defensive watermarks necessarily rely on spectral carriers and introduces a universal spectral attack that succeeds without detector feedback, making it particularly relevant to frequency-based watermarking. Complementary work also considers model-level attacks: Hu et al.~\cite{hu2024stable} show that decoder-rooted watermarks such as Stable Signature can be neutralized by fine-tuning the generator itself.

Finally, recent studies emphasize that removal is only one side of the threat model. Muller et al.~\cite{mullerBlackBoxForgeryAttacks2025} and Sou{\v{c}}ek et al.~\cite{soucek2025transferable} demonstrate realistic black-box forgery and removal attacks for diffusion and post-hoc image watermarks, with the latter corresponding to the WMForger method in our experiments. 
Overall, current evidence suggests that watermark robustness should be assessed not only under conventional distortions but also against steganalytic, regenerative, adaptive, and forgery-oriented attacks.

\section{Proposed Method}

\subsection{Problem Statement}
We consider the following realistic attack setting.

\noindent\textbf{Watermarking process.} Let $\mathcal{W}$ be a watermarking scheme that produces a watermarked image $x^w$ provided a secret key $\kappa$. 
The watermarking process can be integrated into the image generation process, in which case, given a prompt $P$, $x^w=\mathcal{W}(P, \kappa)$.
Another case is a watermark applied post-hoc: given an image without a watermark $x$, $x^w=\mathcal{W}(x, \kappa)$
The watermarking scheme is associated with a verifier $\mathcal{V}$ which, given a query image $x^q$, outputs a probability $p=\mathcal{V}(x^q) \in [0,1]$ of the image being watermarked. The image is considered watermarked if
$p\geq \tau$ for a chosen threshold $\tau$.
\begin{figure*}[htbp]
    \centerline{\includegraphics[width=0.99\textwidth]{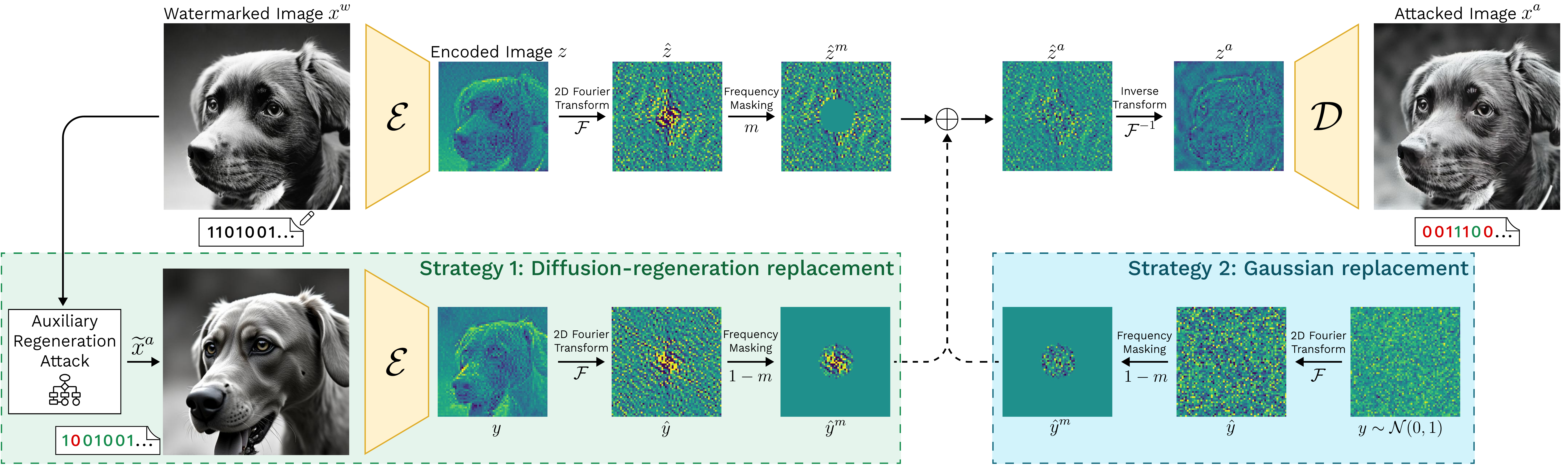}}
    \caption{The overall scheme of the proposed Latent Frequency Masking removal algorithm }
\label{fig:Attack_scheme}
\vspace{-1.3em}
\end{figure*}

\noindent\textbf{Adversary’s objective.} Given a watermarked image $x^w$, the adversary attempts to produce an attacked image $x^a = \mathcal{A}(x^w)$ with the following goals:

\begin{enumerate}
\item Minimize the probability produced by $\mathcal{V}$ until the image is no longer considered watermarked $\mathcal{V}(x^a) < \tau$ 
\item Maximize the perceptual quality of the attacked image relative to the watermarked image $Q(x^a, x^w)$, i.e., minimize the distortion introduced by the attack. Various quality metrics can be selected, e.g., PSNR, SSIM \cite{wang2004image}, LPIPS \cite{zhang2018unreasonable}.
\end{enumerate}

\noindent\textbf{Adversary’s capabilities}
We adopt a threat model in which the adversary has access to:
\begin{enumerate}
\item the watermarked image $x^w$; 
\item a publicly available trained latent diffusion model $(\mathcal{E}, \mathcal{D}, \epsilon_\theta)$ with VAE encoder $\mathcal{E}$, decoder $\mathcal{D}$, and noise prediction network $\epsilon_\theta$ which was used to produce the image $x^w$;
\item the ability to query the public verifier $\mathcal{V}$.
\end{enumerate}

The adversary is assumed not to have access to:
\begin{enumerate}
\item the clean image $x$;
\item the watermark key $\kappa$;
\item the prompt $P$ used during generation; 
\item any internal details of the verifier $\mathcal{V}$, including its architecture or parameters.
\end{enumerate}

\subsection{Algorithm description}
The proposed attack encodes the image using the VAE encoder, computes the spectrum via FFT, performs masking based on a mask and a replacement vector in the frequency region, computes inverse FFT, and produces an outptut image with the VAE decoder. Figure \ref{fig:Attack_scheme} provides an overview of our \textbf{Latent Frequency Masking} attack. We discuss the steps of the algorithm and the rationale behind them below.

\subsubsection{Masking vector generation}
We use a vector which we denote as $\hat{y}$ to replace the watermarked frequencies.
We propose two strategies for $\hat{y}$ generation below.

\noindent\textbf{Gaussian replacement.}
We sample $y$ from the standard Gaussian distribution $\mathcal{N}(0,I)$ and compute 
$\hat{y}=\mathcal{F}(y)$.
This choice is motivated by the fact that during the training of the VAE, an additional loss is used to keep the latent distribution closer to $\mathcal{N}(0, I)$. By using a random vector, we discard the watermark information while keeping the latent ``in-distribution'' for the decoder, thereby ensuring higher-quality decoded images.

\noindent\textbf{Diffusion-regeneration replacement.}
As an image-adaptive replacement strategy, we use diffusion regeneration to construct $\hat{y}$. 
Diffusion regeneration is a common baseline for watermark removal because it partially resamples the image through a generative prior, which can disrupt watermark signals that are not
preserved by the regeneration process.
Let $\tilde{\mathcal{A}}$ denote a diffusion-regeneration attack applied to the
watermarked image, and let
$$\tilde{x}^a = \tilde{\mathcal{A}}(x^w)$$
be the regenerated image. We then encode this image into the latent space and compute its spectrum:
$$y = \mathcal{E}(\tilde{x}^a), \qquad \hat{y}=\mathcal{F}(y).$$

The motivation for this strategy is that diffusion regeneration can suppress watermark information, 
but may also introduce visible changes or alter fine image details. By using its latent representation
only as the replacement source inside the masked frequency region, Latent Frequency Masking can inherit part of its watermark-removal effect while preserving the unmasked coefficients of the original watermarked image.

Using random noise lets us keep the algorithm simple and efficient, however, it essentially discards some information present in the latent. Regeneration attack allows us to retain more information about the semantics of the image, but makes the overall algorithm more complex.

\subsubsection{Mask selection}
\label{sec:mask-selection}

Selecting the correct mask is paramount for the success of the attack.
Kassis et al. \cite{kassis2025unmarker} analyze possible locations for watermarking information based on the requirements of Universal Applicability, Universal Stealthiness, and Robustness to Manipulations. They propose that, based on these requirements, the watermark carrier must reside in the spectral domain. This motivates our choice to perform masking in the Fourier space. They also hypothesize that for "semantic" or generative watermarks, the watermarking information will inhabit the low-frequency areas.

We use a binary mask $m$ to encode the watermark location. 
To disrupt the watermark, the algorithm discards the information in the latent space guided by the mask, rendering the watermark undetectable. At the same time, it leaves a significant portion of the latent intact, thereby retaining the image quality and keeping the attacked image close to the original.

Based on the rationale from \cite{kassis2025unmarker}, we propose the following configuration for possible masks. 
We select a subset of the channels of the latent vector  $S \subseteq \{1,\ldots,C\}$ and define the mask for each of them as follows:
\begin{equation}
m_{ijc} = 
\begin{cases}
    1, &  (i - i_0)^2 + (j - j_0)^2 \leq r^2 \\
    0, &  (i - i_0)^2 + (j - j_0)^2 > r^2    \\
\end{cases}
\end{equation}
where $c \in  S$, $i_0, j_0$ represent the center point of the latent, i.e., $i_0=\lfloor H/2\rfloor$, $j_0=\lfloor W/2 \rfloor$, and $r$ is the radius of the mask, which we consider a parameter. For channels that were not selected, $c \notin  S$, we set the mask to $0$.
The final mask is defined by its radius and selected channels.
These parameters define a relatively small search space; in our experiments, we use exhaustive search to establish the optimal configuration for each of the tested watermarking algorithms using TPR@1\%FPR as the selection objective. We select $r\in\{10,15\}$ and $S$ as every single channel and every pair of channels, which gives us $2*(4+6)=20$ variants to search; we consider this number acceptable in realistic scenarios. 


\subsubsection{Algorithm}

The first step of the attack consists of producing the latent vector $z^w$ corresponding to $x^w$ via the VAE encoder $z^w=\mathcal{E}(x^w)\in\mathbb{R}^{H \times W \times C}$. The next step is computing the channel-wise 2-dimensional Fourier transform $\hat{z}^w = \mathcal{F}(z^w) \in \mathbb{C}^{H \times W \times C}$

Given the binary mask $m\in \mathbb{B}^{H \times W \times C}$ (where $\mathbb{B}=\{0,1\}$), a masking vector $\hat{y} \in \mathbb{C}^{H \times W \times C}$, and a parameter $\beta\in[0,1]$, we perform partial masking of a part of the latent. 

The new latent is produced via

\begin{equation}
\label{eq:1}
\hat{z}^a_{ijc} =
\begin{cases}
    \beta \hat{z}^w_{ijc}+ (1-\beta)\hat{y}_{ijc}, &  m_{ijc} = 1 \\
    \hat{z}^w_{ijc}, &  m_{ijc} = 0
\end{cases}
\end{equation}

The parameter $\beta$ controls the strength of the replacement inside the masked frequency region. When $\beta=1$, the masked coefficients remain equal to the original latent coefficients; when $\beta=0$, the masked coefficients are fully replaced. Intermediate values interpolate between these two cases. 
Coefficients outside the mask are independent of $\beta$ and remain unchanged.

Finally, we compute the inverse Fourier transform 
$$z^a = \mathcal{F}^{-1}(\hat{z}^a)$$
and recover the resulting image using the VAE decoder
$$x^a = \mathcal{D}(z^a)$$


The complete algorithm is presented in Algorithm \ref{alg:latent_frequency_masking}.
For the Gaussian replacement strategy, the attack requires a single forward pass through the VAE encoder and decoder, Gaussian noise sampling, and forward and inverse Fourier transform calculation, which makes it computationally efficient. 
Diffusion-regeneration replacement strategy involves diffusion calculation, an additional encoder pass, and a Fourier transform.
Our algorithm does not employ any gradient computation or optimization steps.
The code for the algorithm is available at \url{https://github.com/alvitrioliks/lfm}.

\begin{algorithm}
\caption{Latent frequency masking}
\label{alg:latent_frequency_masking}
\begin{algorithmic}[1]
\Require Watermarked image $x^w$, encoder $\mathcal{E}$ and decoder $\mathcal{D}$ from a Latent Diffusion Model with latent dimension $K=H \times W \times C$, mask $m\in \mathbb{B}^{H \times W \times C}$, parameter $\beta$, masking mode $s$
\Require Existing attack $\tilde{\mathcal{A}}$ if $s=\textsc{Attack}$
\Ensure Attacked image $x^a$

\State $z^w \gets \mathcal{E}(x^w)$ \Comment{Encode to latent space}
\State $\hat{z}^w \gets \mathcal{F}(z^w)$ \Comment{Fourier transform}

\If{$s=\textsc{Gaussian}$}
    \State $y \sim \mathcal{N}(0,I), \quad y\in\mathbb{R}^{H\times W\times C}$ \Comment{Sample Gaussian noise in latent space}
    \State $\hat{y} \gets \mathcal{F}(y)$ \Comment{Fourier transform the noise}
\ElsIf{$s=\textsc{Diffusion}$}
    \State $\tilde{x}^a \gets \tilde{\mathcal{A}}(x^w)$ \Comment{Produce an attacked image}
    \State $y \gets \mathcal{E}(\tilde{x}^a)$ \Comment{Encode to latent space}
    \State $\hat{y} \gets \mathcal{F}(y)$ \Comment{Fourier transform}
\EndIf

\For{$i \gets 1$ to $H$, $j \gets 1$ to $W$, $c \gets 1$ to $C$}
    \If{$m_{ijc} = 1$}
        \State $\hat{z}^a_{ijc} \gets \beta \hat{z}^w_{ijc}+ (1-\beta)\hat{y}_{ijc}$
    \Else
        \State $\hat{z}^a_{ijc} \gets \hat{z}^w_{ijc}$
    \EndIf
\EndFor

\State $z^a \gets \mathcal{F}^{-1}(\hat{z}^a)$ \Comment{Inverse Fourier transform}
\State $x^a \gets \mathcal{D}(z^a)$ \Comment{Decode to pixel space}
\State \Return $x^a$
\end{algorithmic}
\end{algorithm}

\subsection{Theoretical analysis}
This section provides a theoretical analysis of the distortion introduced by the attack. We prove that the distortion is bounded and depends on the parameters of the algorithm.


\begin{theorem}[Distortion bound for latent frequency masking]
\label{thm:lfm_bound}
Let $x^w \in \mathbb{R}^{H_I \times W_I \times 3}$ be a watermarked image generated by a latent diffusion model $(\mathcal{E}, \mathcal{D}, \epsilon_\theta)$, and $z^w=\mathcal{E}(x^w)\in\mathbb{R}^{H\times W\times C}$ be the latent produced from $x^w$. 

Let $\mathcal{F}$ denote a channel-wise 2-dimensional unitary discrete Fourier transform on the latent space, and write
$\hat{z}^w = \mathcal{F}(z^w)$.

Let $m \in \mathbb{B}^{H \times W \times C}$ be a binary frequency mask,
let $\hat{y} \in \mathbb{C}^{H \times W \times C}$ be the masking vector and let $\hat{z}^{a} \in \mathbb{C}^{H \times W \times C}$ be the attacked latent frequency coefficients defined in \eqref{eq:1},

Let $z^a = \mathcal{F}^{-1}(\hat{z}^a)$ and $x^a = \mathcal{D}(z^a)$ be the attacked latent and the attacked image produced by Algorithm \ref{alg:latent_frequency_masking}.

Assume that the decoder $\mathcal{D}$ is $L$-Lipschitz with respect to the Euclidean norm, i.e.,
$$
\|\mathcal{D}(z)-\mathcal{D}(z')\|_2 \leq L \|z-z'\|_2
\quad \text{for all } z,z' \in \mathbb{R}^{H \times W \times C}.
$$
Then the pixel-space distortion introduced by latent frequency masking satisfies

$$
\|x^w-x^a\|_2 \leq
\|x^w-\mathcal{D}(\mathcal{E}(x^w))\|_2 +
L \sqrt{ \sum_{i,j,c:m_{ijc}=1} \left| \hat{z}_{ijc}^w-\hat{z}_{ijc}^a \right|^2}.
$$

\end{theorem}

\begin{proof}

Since $x^a=\mathcal{D}(z^a)$, and using the triangle inequality 

\begin{equation}
\begin{split}
\label{proof_eq_1}
\|x^w-x^a\|_2 
&= 
\|x^w-\mathcal{D}(z^a)\|_2 \\
&\leq
\|x^w-\mathcal{D}(z^w)\|_2 +  \|\mathcal{D}(z^w)-\mathcal{D}(z^a)\|_2
\end{split}
\end{equation}

From the Lipschitz property
$$\|\mathcal{D}(z^w)-\mathcal{D}(z^a)\|_2 \leq L\|z^w-z^a\|_2$$

Using a well-known property of the Fourier transform known as the Parseval theorem or the Plancherel theorem
$$L\|z^w-z^a\|_2 = L\|\mathcal{F}(z^w-z^a)\|_2$$

Note that the output of the Fourier transform is complex, and the complex $l_2$ norm is used for the right-hand side.

Since the Fourier transform is a linear transformation
$$L\|\mathcal{F}(z^w-z^a)\|_2 = L\|\mathcal{F}(z^w)-\mathcal{F}(z^a)\|_2 = L\|\hat{z}^w-\hat{z}^a\|_2=$$

By the definition of the $l_2$ norm
$$= L\sqrt{\sum_{i}^{H}\sum_{j}^{W}\sum_{c}^{C}\left|\hat{z}_{ijc}^w-\hat{z}_{ijc}^a\right|^2}=$$

Notice that $\hat{z}_{ijc}^w = \hat{z}_{ijc}^a$ for $i,j,c$ such that  $m_{ijc}\neq1$, therefore
\begin{equation}
\label{proof_eq_2}
= L\sqrt{\sum_{i,j,c: m_{ijc}=1 }\left|\hat{z}_{ijc}^w-\hat{z}_{ijc}^a\right|^2}
\end{equation}
Combining \eqref{proof_eq_1} and \eqref{proof_eq_2}, and using $z^w=\mathcal{E}(x^w)$ we get the desired bound.

\end{proof}

\noindent\textbf{Discussion} The theorem states that the distortion introduced into the watermarked image by the attack process is bounded by the sum of the VAE reconstruction error and the distance between the latents in the region of frequencies defined by the mask.

The size of the mask directly impacts the distortion: the smaller the masked region, the lower the distortion introduced. This insight guides the selection of mask radius in our experiments: we want to select the smallest radius that allows us to erase the watermarking information.

The choice of the vector $\hat{y}$ also affects the distortion: the closer it is to $\hat{z}^w$, the better. This helps to justify the design of the Diffusion replacement strategy: since the regenerated image is similar to the original, the latents will also be similar, and therefore the distortion will be lower.

\noindent\textbf{The Lipschitz continuity of the decoder}
One of the assumptions in Theorem~\ref{thm:lfm_bound} is that the VAE decoder is Lipschitz continuous. 
This assumption is natural for the Stable Diffusion decoder used in our experiments. 
After its weights are fixed, the decoder is a deterministic map $\mathcal{D}:\mathbb{R}^{H \times W \times C}\rightarrow\mathbb{R}^{H_I\times W_I\times 3}$ implemented by a neural network. 
The network is built from convolutional and linear layers, pointwise activation functions, upsampling operations, residual connections, attention blocks, and normalization layers.
The composition of Lipschitz maps is also Lipschitz, with a constant bounded by the product of the constituent constants (and by corresponding sums for residual branches). 
Therefore, if all the building blocks are Lipschitz under the Euclidean norm, then the decoder also admits some finite Lipschitz constant $L$. 

Prior work gives bounds and estimation procedures for fully connected layers, convolutional layers, pooling layers, activation functions, normalization, and residual connections
under $\ell_2$, $\ell_1$, and $\ell_\infty$ norms~\cite{gouk2021regularisation}. For attention blocks, however, the situation is more complex: by itself, the standard dot-product self-attention is not globally Lipschitz for an unbounded input domain \cite{kim2021lipschitz}.
However, it can be shown that together with normalization layers, the combined map is Lipschitz. In particular, \cite{castin2023smooth} proves that for inputs of length $n$ restricted to any compact set, self-attention has an explicit finite Lipschitz upper bound. In the considered architecture, GroupNorm is applied before the query, key, and value projections, and since GroupNorm normalizes each channel group using variance, its output is bounded. Hence, the attention module only receives inputs in a bounded set, and the bounded-domain Lipschitz result applies. The combined map is therefore Lipschitz with a constant obtained by composing the GroupNorm, attention, and residual constants.

We note that, in practice, the Lipschitz constant depends on the weights and architecture of the specific decoder network.
For a deep decoder, the worst-case bound may be somewhat loose, because the product of layer constants can be large. 
We do not rely on the numerical value of $L$ in our attack; the assumption is used only to justify that small latent perturbations imply bounded changes in the decoded image.

\section{Experiments}
The prompts for evaluation are drawn from two publicly available datasets, namely \href{https://huggingface.co/datasets/poloclub/diffusiondb}{DiffusionDB} \cite{wangDiffusionDBLargescalePrompt2022} and \href{https://huggingface.co/datasets/whyen-wang/coco_captions}{MS-COCO} \cite{lin2015microsoftcococommonobjects}, 1000 prompts are selected from each. The corresponding watermarked images are generated from the prompts using the respective watermarking methods and their underlying diffusion model (all methods employ Stable Diffusion~v2.1~\cite{sd_cite} as a base model). We perform the parameter search needed for LFM on a set of 100 watermarked images generated from DiffusionDB prompts separate from the main evaluation set. 

\begin{figure*}[htbp]
    \centerline{\includegraphics[width=0.99\textwidth]{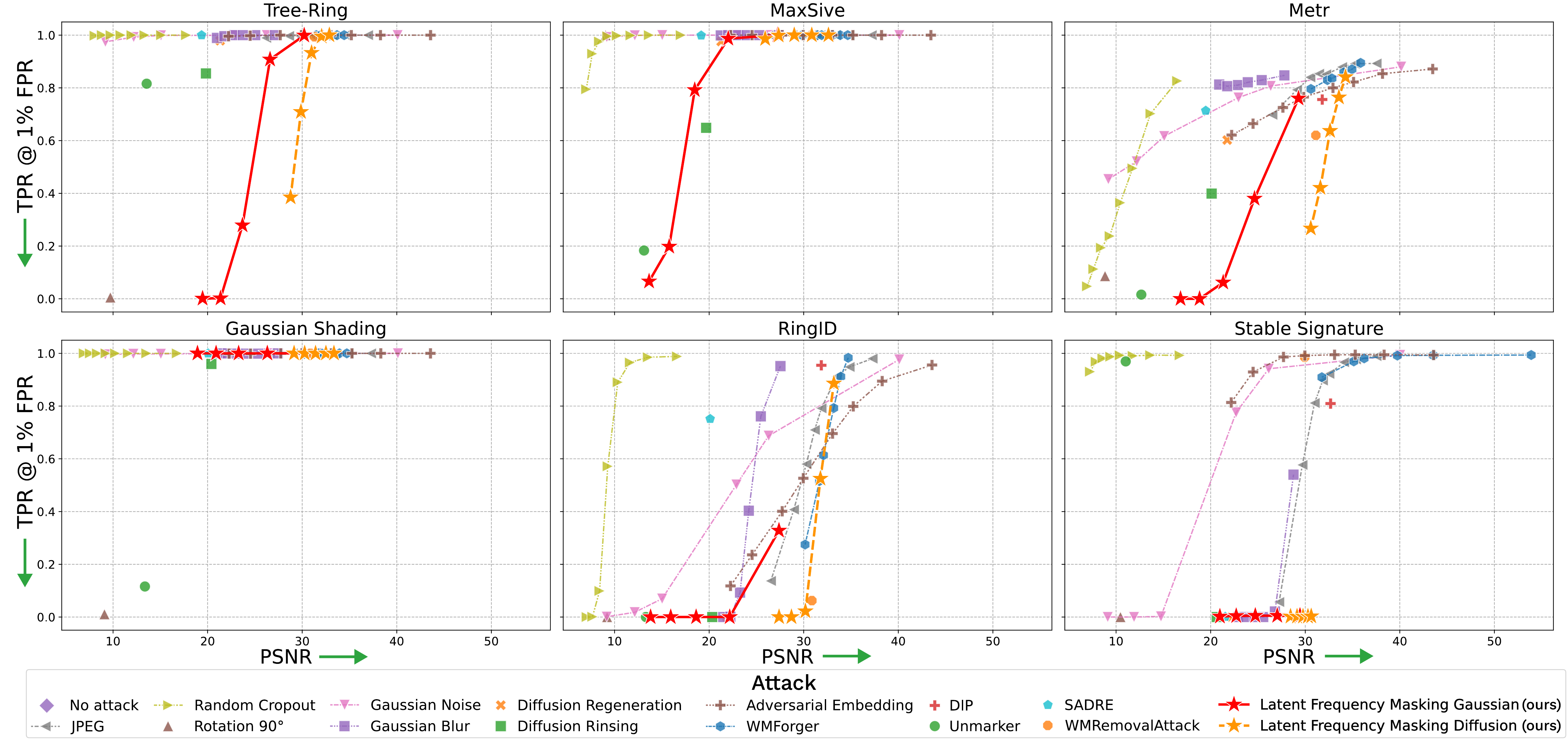}}
    \caption{Removal efficiency~(TPR@1\%FPR$\downarrow$ of the watermark after attack) vs image quality~(PSNR$\uparrow$) for different attacks measured on DiffusionDB dataset. Lower-right is better; LFM variants sit on the Pareto front for all methods except Gaussian Shading}
\label{fig:tpr_psnr_diffdb}
\end{figure*}

\subsection{Compared methods}

We evaluate our algorithm against six representative diffusion watermarking methods, namely Stable Signature \cite{fernandez2023stable}, Tree-Ring \cite{wen2023tree}, RingID \cite{ci2024ringid}, Gaussian Shading \cite{yang2024gaussian}, Maxsive \cite{mao2025maxsive}, and METR \cite{varlamov2024metr}. 

We perform a comparison with image distortions which are known to hinder watermark performance, such as Gaussian Blur, Gaussian Noise, JPEG compression, Rotation, and Random Cropout, as well as several existing attack methods, including 
Adversarial Embedding \cite{an2024waves}, Diffusion Regeneration \cite{zhaoInvisibleImageWatermarks2024}, Diffusion Rinsing \cite{an2024waves}, DIP \cite{liangBaselineMethodRemoving2025}, SADRE \cite{alam2025saliency}, WMForger \cite{soucek2025transferable}, WMRemovalAttack \cite{mullerBlackBoxForgeryAttacks2025} and Unmarker \cite{kassis2025unmarker}
to comprehensively evaluate the effectiveness of attacks and the quality of the attacked images. We select the FLUX \cite{flux2024} model for Diffusion Regeneration and Diffusion Rinsing, as well as for the Diffusion replacement strategy in LFM.

\subsection{Evaluation metrics}

We evaluate watermark erasing attacks from two perspectives: (i) how reliably the watermark can be removed using the given transformation, and (ii) how much this distortion affects the image quality.

The first perspective concerns the reliability of watermark extraction. 
Reliability evaluation depends on whether a method embeds only a detectable signal or a full binary payload. In zero-bit watermarking, the detector decides whether a watermark is present, yielding a binary detection setting.
Practical systems are usually expected to operate under strict false-alarm constraints; therefore, robustness is commonly summarized by the True Positive Rate achieved at a predefined low False Positive Rate, i.e., TPR@x\%FPR, we use an FPR of 1\%.

For methods that embed multi-bit messages, extraction quality is determined by the similarity between the recovered and original payloads. 
Multi-bit extraction can also be interpreted as a detection problem by using the agreement between the embedded and decoded messages as a detection score. This makes it possible to compare payload-based methods with zero-bit methods using the same operating-point metrics, including TPR@x\%FPR.

The second perspective focuses on the image quality. When a reference image is available, the distortion caused by the erasing method can be measured with full-reference metrics such as PSNR, SSIM~\cite{wang2004image}, and LPIPS~\cite{zhang2018unreasonable}. They evaluate pixel-level correspondence between the distorted and reference images. In settings where no clean reference is available, perceptual quality can instead be estimated with no-reference image quality metrics. To this end, we employ the popular CLIP-IQA~\cite{wang2023exploring} model. 

\section{Results}
\subsection{Attack Performance}

\begin{table*}[t]
\caption{Attack performance and quality metrics averaged across all tested watermarks on DiffusionDB and MS-COCO datasets}
\centering
\resizebox{\textwidth}{!}{


\begin{tabular}{lccccc|ccccc}
\toprule
& \multicolumn{5}{c}{DiffusionDB} & \multicolumn{5}{c}{MS-COCO} \\
\cmidrule(lr){2-6} \cmidrule(lr){7-11}
& Robustness 
& \multicolumn{4}{c}{Quality metrics}
& Robustness
& \multicolumn{4}{c}{Quality metrics} \\
\cmidrule(lr){3-6} \cmidrule(lr){8-11}
Attack
& TPR@1\%FPR $\downarrow$ & PSNR $\uparrow$ & SSIM $\uparrow$ & LPIPS $\downarrow$ & CLIP IQA $\uparrow$
& TPR@1\%FPR $\downarrow$ & PSNR $\uparrow$ & SSIM $\uparrow$ & LPIPS $\downarrow$ & CLIP IQA $\uparrow$ \\
\midrule
No attack
& 0.97 & $\infty$ & 1.000 & 0.000 & 0.866
& 0.98 & $\infty$ & 1.000 & 0.000 & 0.883 \\

Adv. Embedding \cite{an2024waves}
& 0.71 & 22.2 & 0.462 & 0.402 & 0.349
& 0.71 & 22.3 & 0.459 & 0.439 & 0.301 \\

Diff. Regeneration \cite{zhaoInvisibleImageWatermarks2024}
& 0.51 & 21.7 & 0.620 & 0.205 & 0.840
& 0.51 & 22.4 & 0.657 & 0.178 & 0.833 \\

Diff. Rinsing \cite{an2024waves}
& 0.42 & 20.1 & 0.553 & 0.258 & 0.856
& 0.44 & 20.8 & 0.591 & 0.223 & 0.838 \\

DIP \cite{liangBaselineMethodRemoving2025}
& 0.86 & 31.9 & 0.929 & 0.105 & 0.714
& 0.91 & 32.3 & 0.942 & 0.099 & 0.771 \\

SADRE \cite{alam2025saliency}
& 0.69 & 19.9 & 0.733 & 0.114 & 0.860
& 0.69 & 21.2 & 0.777 & 0.109 & 0.848 \\

Unmarker \cite{kassis2025unmarker}
& 0.35 & 12.8 & 0.274 & 0.419 & 0.699
& 0.38 & 12.9 & 0.281 & 0.402 & 0.664 \\

WMForger \cite{soucek2025transferable}
& 0.85 & 32.3 & 0.891 & 0.143 & 0.629
& 0.87 & 33.1 & 0.906 & 0.129 & 0.698 \\

WMRemovalAttack \cite{mullerBlackBoxForgeryAttacks2025}
& 0.73 & 30.8 & 0.837 & 0.185 & 0.746
& 0.73 & 31.1 & 0.845 & 0.213 & 0.656 \\

LFM Gaussian (ours)
& 0.36 & 21.6 & 0.802 & 0.119 & 0.857
& 0.32 & 21.2 & 0.817 & 0.129 & 0.887 \\

LFM Diffusion (ours)
& 0.44 & 28.4 & 0.897 & 0.043 & 0.848
& 0.44 & 28.9 & 0.908 & 0.045 & 0.887 \\
\bottomrule
\end{tabular}
}
\vspace{-1.5em}
\label{tab:tpr_metrics_table}
\end{table*}

Figure \ref{fig:tpr_psnr_diffdb} reports TPR@1\%FPR versus image quality, measured by PSNR, for the six watermarking methods.
For Latent Frequency Masking, each point corresponds to a specific $\beta$ parameter; we select $\beta = 0, 0.2, 0.4, 0.6, 0.8$. Mask radius and channels are selected for each watermarking method using the method described in 
Section \ref{sec:mask-selection}, namely $c=1$, $r=10$ for Tree-Ring, Gaussian Shading, and Stable Signature, $c=4$, $r=10$ for METR and $c=\{1,4\}$, $r=15$ for Maxsive and RingID. We also vary parameters for several attacks, namely compression strength for JPEG, filter size for Gaussian blur, noise $\sigma$ for Gaussian noise, crop size for Random Cropout, size of the perturbation $\varepsilon$ for adversarial embedding, and learning rate for WMForger.

The results show that Latent Frequency Masking~(LFM) achieves strong attack performance across most methods, achieving
TPR@1\%FPR close to $0$ for Stable Signature, Tree-Ring, RingID, Maxsive, and METR for certain $\beta$ parameters.
Notably, LFM is the only method capable of reliably erasing the Tree-Ring watermark within reasonable quality constraints. 

Latent Frequency Masking with Diffusion replacement strategy achieves higher image quality compared to the Gaussian strategy, with several dB PSNR improvement in exchange for increased complexity. For MaxSive, however, it fails to erase the watermark; this could be explained by the fact that MaxSive is resistant to Diffusion Regeneration, using regenerated values for the replacement introduces watermarking information back into the latent.

For Gaussian Shading, almost all attacks except for Rotation and Unmarker are unable to bypass the watermark, and even they perform poorly in terms of quality. The lack of success of the Latent Frequency Masking attack can be explained by the fact that Gaussian Shading, by design, embeds the watermarking information throughout the entire latent, rendering partial masking ineffective.

Figure \ref{fig:tpr_lpips_diffdb} shows attack performance versus image quality according to the LPIPS metric. It can be seen that, in terms of LPIPS and PSNR, LFM performs similarly. The Diffusion replacement strategy still achieves better quality than Gaussian; however, the improvement is less pronounced. For Stable Signature, LFM outperforms all other attacks in LPIPS, achieving almost a perfect removal rate across all quality levels.

Figures \ref{fig:tpr_psnr_diffdb} and \ref{fig:tpr_lpips_diffdb} show that in most cases, LFM achieves Pareto-optimal performance among tested methods in terms of image quality-efficiency tradeoff.


Table \ref{tab:tpr_metrics_table} presents performance and quality metrics of attacks averaged across all six watermarking methods on DiffusionDB and MS-COCO datasets. We select $\beta=0.4$ for LFM Gaussian and $\beta=0$ for LFM Diffusion since they achieve a good trade-off in quality and performance. LFM Gaussian, together with Unmarker achieve the lowest TPR@1\%FPR score. For LFM Diffusion TPR@1\%FPR is slightly worse due to its poor performance on MaxSive; the quality metrics are significantly higher. Both LFM variants also show high quality according to all 4 measured quality metrics, placing close to the top among the measured algorithms.

These findings show that Latent Frequency Masking not only successfully bypasses watermark verification in most cases but also, on average, outperforms current attacks in terms of overall performance-quality tradeoff.

\subsection{Attack Speed Comparison}
Figure~\ref{fig:tpr_vs_time} presents a plot of attack performance vs the computation times of attacks that were used for comparison in this paper. The timings were measured on a server with an NVIDIA RTX A6000 GPU and averaged over the 1000 images from the DiffusionDB dataset. We do not include time for parameter search for LFM since it is performed once per watermarking algorithm and the selected parameters are reused for each tested image.
Latent Frequency Masking with Gaussian replacement strategy achieves the highest attack speed among the compared algorithms while having comparable or better attack performance. For Latent Frequency Masking with the Diffusion replacement strategy, speed is predictably comparable to Diffusion Regeneration; however, it's notable that it performs slightly better in terms of quality and performance.

\begin{figure*}[htbp]
    \centerline{\includegraphics[width=0.99\textwidth]{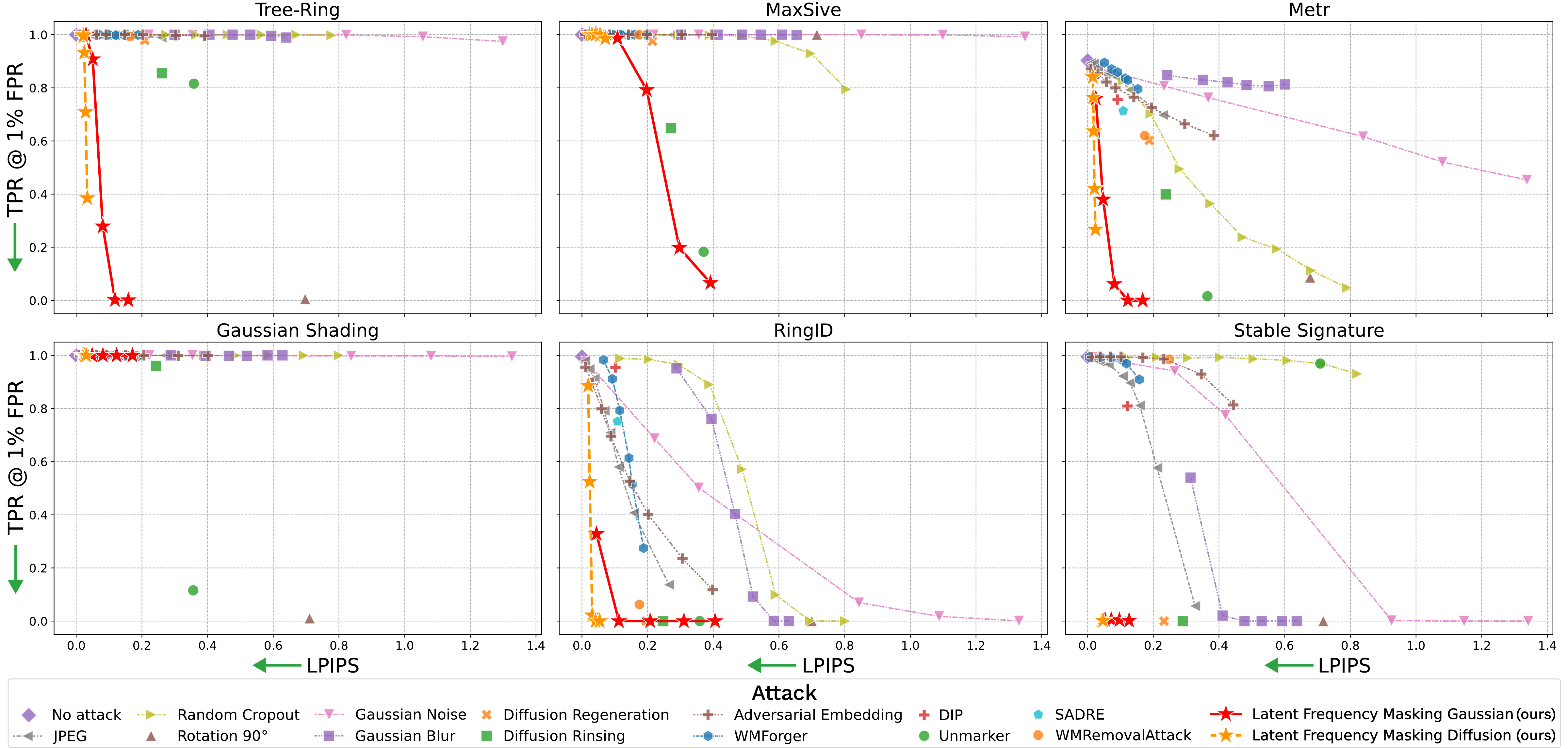}}
    \caption{Removal efficiency~(TPR@1\%FPR$\downarrow$ of the watermark after attack) vs image quality~(LPIPS$\uparrow$) for different attacks measured on DiffusionDB dataset. Lower-left is better; LFM variants sit on the Pareto front for all methods except Gaussian Shading}
\label{fig:tpr_lpips_diffdb}
\vspace{-1.3em}
\end{figure*}

\begin{figure}[tbp]
    \centerline{\includegraphics[width=0.9\columnwidth]{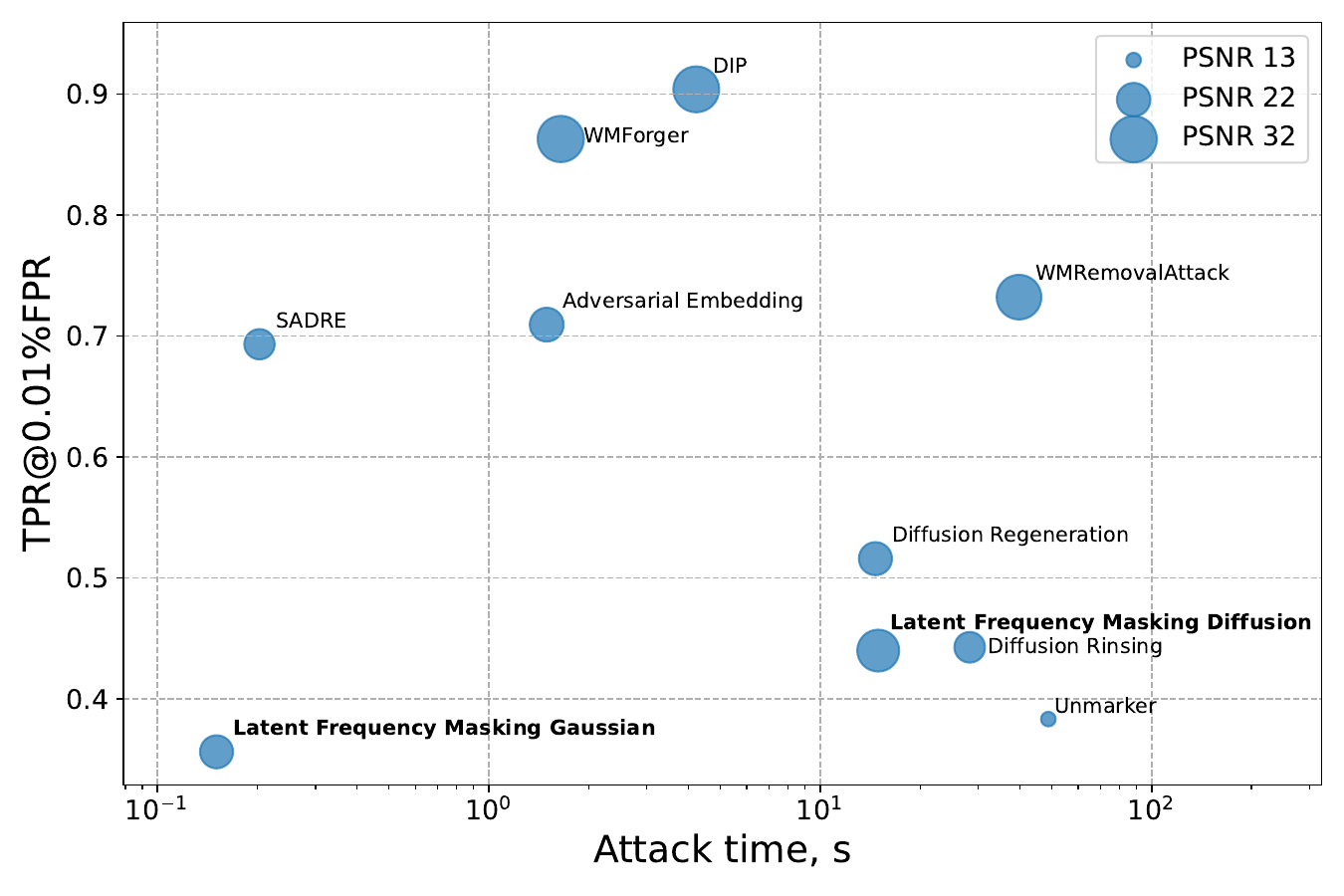}}
    \caption{Average attack performance vs execution time. The size of the dot represents the average PSNR on the DiffusionDB dataset}
\vspace{-1.5em}
\label{fig:tpr_vs_time}
\end{figure}

\subsection{Visual Comparison}
Figure~\ref{fig:visual_comp} presents a visual comparison of several attack methods. It can be seen that some of the methods introduce visible artifacts or noticeably change image details. LFM mostly avoids this; however, we note that it still disturbs color in some parts of the image. For the presented image, only our method is able to erase the watermark information.

\begin{figure*}[htbp]
    \centerline{\includegraphics[width=0.99\textwidth]{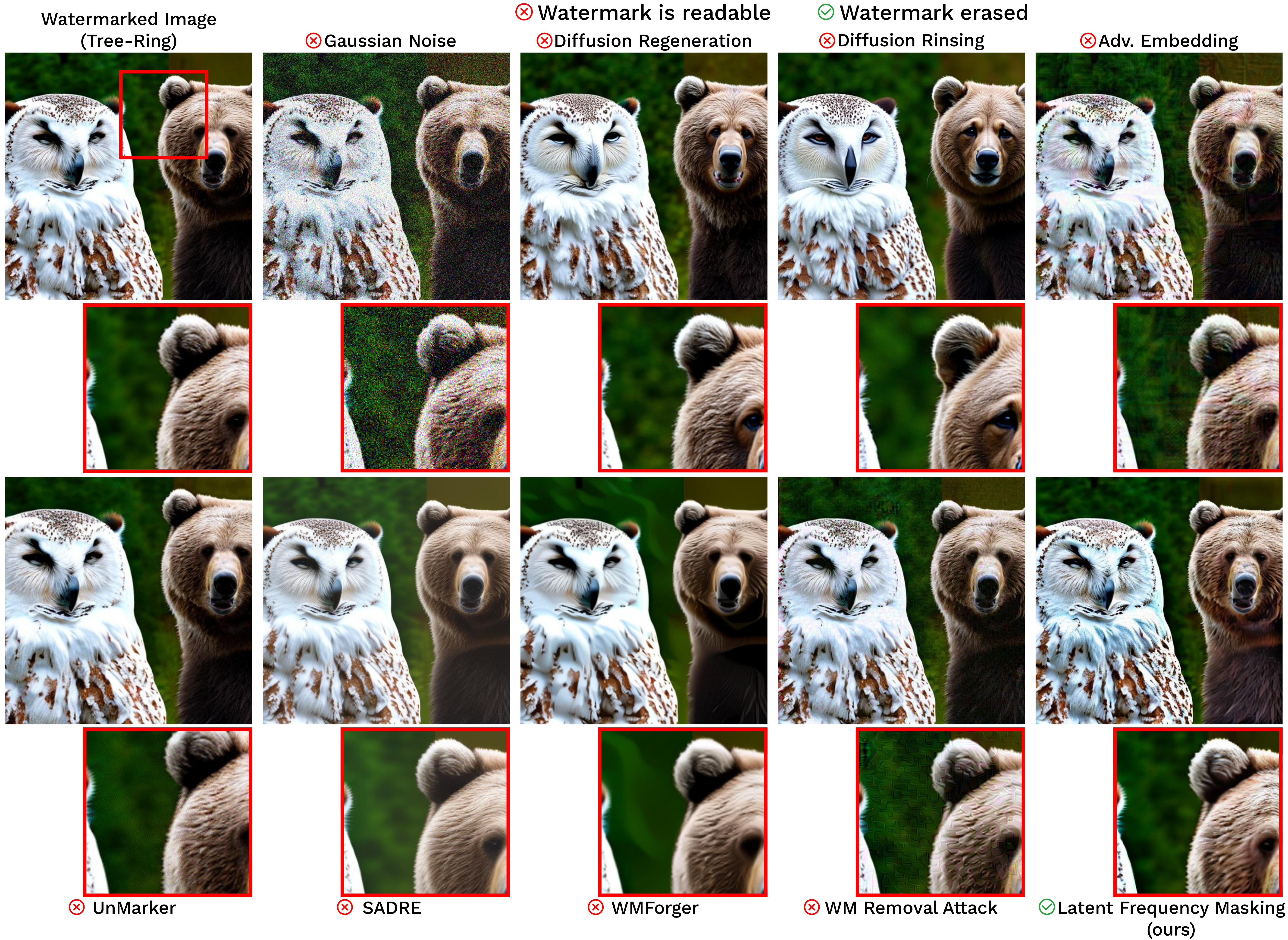}}
    \caption{Examples of a Tree-Ring-watermarked image subjected to different attacks. The prompt is sampled from the DiffusionDB database.}
\label{fig:visual_comp}
\vspace{-1.5em}
\end{figure*}



\subsection{Ablation Study}
We perform an ablation study on all of the steps of the proposed algorithm on the DiffusionDB dataset to assess their importance. Since the algorithm is composed of consecutive transformations, i.e., VAE, FFT, and masking, we perform attacks with one or more of these transformations disabled. We use Tree-Ring watermark and set $\beta=0$.
Results are presented in Table \ref{tab:ablation}. All variants of the attack with excluded steps fail to erase the watermarking information, even though some of them show better PSNR scores. We conclude that all steps of the algorithm are necessary for watermark removal.

\begin{table}[t]
\caption{Ablation results, VAE represents only the Encoder and Decoder steps without any masking, VAE + Masking: masking performed in spatial domain without FFT step, FFT + Masking: frequency masking is performed on images without Encoder and Decoder steps, VAE + FFT + Masking: full algorithm }
\centering

\begin{tabular}{lcc}
\toprule
Attack                                & TPR@1\%FPR   & PSNR \\
\midrule
VAE                                   & 1.0          & 33.1  \\        
VAE + Masking                         & 1.0          & 23.7  \\
FFT + Masking                         & 0.995        & 11.4  \\
VAE + FFT + Masking (Gaussian)        & 0.001        & 19.4  \\
VAE + FFT + Masking (Diffusion)       & 0.385        & 28.7  \\
\bottomrule
\end{tabular}

\label{tab:ablation}
\vspace{-1.5em}
\end{table}

\section{Discussion and Limitations}
In this paper, we propose a novel attack on watermarks for diffusion-generated images. 
The method operates by replacing selected Fourier coefficients in the VAE latent representation and supports Gaussian and diffusion-regeneration replacement strategies. By modifying only selected frequency coefficients of the VAE latent representation, Latent Frequency Masking can substantially reduce watermark detectability while preserving most of the visual content of the original image. The Gaussian replacement variant
is particularly efficient, requiring only VAE encoding and decoding together with Fourier transforms. The diffusion-regeneration variant improves image quality in several settings by using an
image-adaptive replacement latent, although its effectiveness depends on whether regeneration itself suppresses the target watermark.

The experiments reveal that Tree-Ring, RingID, Stable Signature, MaXsive, and METR are vulnerable to Latent Frequency Masking under
suitable mask and replacement parameters. In contrast, Gaussian Shading is significantly harder to remove with partial masking, which is consistent with its design: watermark information is
distributed throughout the latent rather than concentrated in a localized spectral region. This suggests that localizing watermark information in predictable latent-frequency regions may
create an exploitable weakness, while distributed embedding can improve resistance to this class of attacks.

The proposed attack has several limitations. First, our mask selection procedure uses search over a small parameter space and therefore may not be
optimal for an adversary with a very strict query budget. Second, the attack assumes access to a compatible latent diffusion autoencoder, which is realistic for many open generative
pipelines but may not hold for proprietary systems. Finally, our theoretical bound characterizes the distortion introduced by the latent modification, but does not by itself guarantee watermark removal; removal depends on how each
watermarking scheme stores and detects its signal.

From a defensive perspective, the findings indicate that watermark evaluation should include latent-space and frequency-space attacks in addition to standard image distortions. Robust watermarking schemes should avoid placing most of the watermark evidence in compact, predictable latent-frequency regions. Promising directions include detector designs that remain sensitive after partial latent replacement, and training procedures that explicitly include Latent Frequency Masking as an adversarial transformation.

\section{Conclusion}
In conclusion, we introduced Latent Frequency Masking, a simple and efficient attack for evaluating the robustness of invisible watermarks in diffusion-generated images. Across multiple watermarking schemes, LFM achieves a favorable trade-off between watermark removal, perceptual quality, and computational cost, exposing weaknesses that are not captured by conventional robustness tests. We hope these results encourage stronger evaluation protocols and more robust watermark designs for generative image provenance.

\section*{Acknowledgment}

The work of K. Abud, A. Yakushev, A. Akimenkov, D. Obydenkov, Y. Markin and A. Antsiferova was supported by a grant, provided by the Ministry of Economic Development of the Russian Federation (agreement dated June 20, 2025 No. 139-15-2025-011, identifier 000000C313925P4G0002). 
The research was carried out using the MSU-270 supercomputer of Lomonosov Moscow State University.

\bibliographystyle{IEEEtran}
\bibliography{reference}

\end{document}